\def\year{2022}\relax
\documentclass[letterpaper]{article} 
\usepackage{aaai22}  
\usepackage{times}  
\usepackage{helvet}  
\usepackage{courier}  
\usepackage[hyphens]{url}  
\usepackage{graphicx} 
\usepackage{natbib}  
\usepackage{caption} 
\DeclareCaptionStyle{ruled}{labelfont=normalfont,labelsep=colon,strut=off} 
\usepackage{algorithm}
\usepackage{algorithmic}

\usepackage{amsthm}
\newtheorem{theorem}{Theorem} 

\newtheorem{corollary}{Corollary}
\newtheorem{remark}{Remark}
\usepackage{amsmath}
\usepackage{multirow}
\usepackage{bbold}

\usepackage{newfloat}
\usepackage{listings}
\floatstyle{ruled}
\newfloat{listing}{tb}{lst}{}
\floatname{listing}{Listing}
\title{On the Efficiency of Subclass Knowledge Distillation in Classification Tasks}

\author{
Ahmad Sajedi, Konstantinos N. Plataniotis\\
}
\affiliations {
The Edward S. Rogers Sr. Department of Electrical \& Computer Engineering, University of Toronto\\
ahmad.sajedi@mail.utoronto.ca, kostas@ece.utoronto.ca
}

\begin{document}
\maketitle
\frenchspacing
\begin{abstract}
This work introduces a novel knowledge distillation framework for classification tasks where information on existing subclasses is available and taken into consideration. In classification tasks with a small number of classes or binary detection (two classes) the amount of information transferred from the teacher to the student network is restricted, thus limiting the utility of knowledge distillation. Performance can be improved by leveraging information about possible subclasses within the available classes in the classification task. To that end, we propose the so-called Subclass Knowledge Distillation (SKD) framework, which is the process of transferring the subclasses' prediction knowledge from a large teacher model into a smaller student one. Through SKD, additional meaningful information which is not in the teacher's class logits but exists in subclasses (e.g., similarities inside classes) will be conveyed to the student and boost its performance. Mathematically, we measure how many extra information bits the teacher can provide for the student via SKD framework. The framework developed is evaluated in clinical application, namely colorectal polyp binary classification. It is a practical problem with few original classes and a number of subclasses per class. In this application, clinician-provided annotations are used to define subclasses based on the annotation label's variability in a curriculum style of learning. A lightweight, low complexity student trained with the proposed framework achieves an F1-score of $85.05\%$, an improvement of $2.14\%$ and $1.49\%$ gain over the student that trains without and with conventional knowledge distillation, respectively. These results show that the extra subclasses' knowledge (i.e., $0.4656$ label bits per training sample in our experiment) can provide more information about the teacher generalization, and therefore SKD can benefit from using more information to increase the student performance.
\end{abstract}

\section{Introduction} \label{sec1}
In many real-world classification problems, each labeled class has a number of available semantically meaningful subclasses. For example, in the cancer diagnosis task, which involves the detection of benign and abnormal lesions, the abnormal class may have multiple subclasses in which each of them can express different types or organs of cancer disease \cite{oakden2020hidden, mlynarski2019deep}. Models trained exclusively on class labels often ignore the fine-grained knowledge of subclasses, which can have an effect on model training, particularly for clinical tasks such as cancer detection \cite{sohoni2020no, oakden2020hidden}. We can take advantage of this subclass knowledge by forcing the teacher model to train the subclass labels. Then the knowledge can be transferred from the teacher to the student network.

\begin{figure} 
    \centering
    \includegraphics[width= 0.85\linewidth]{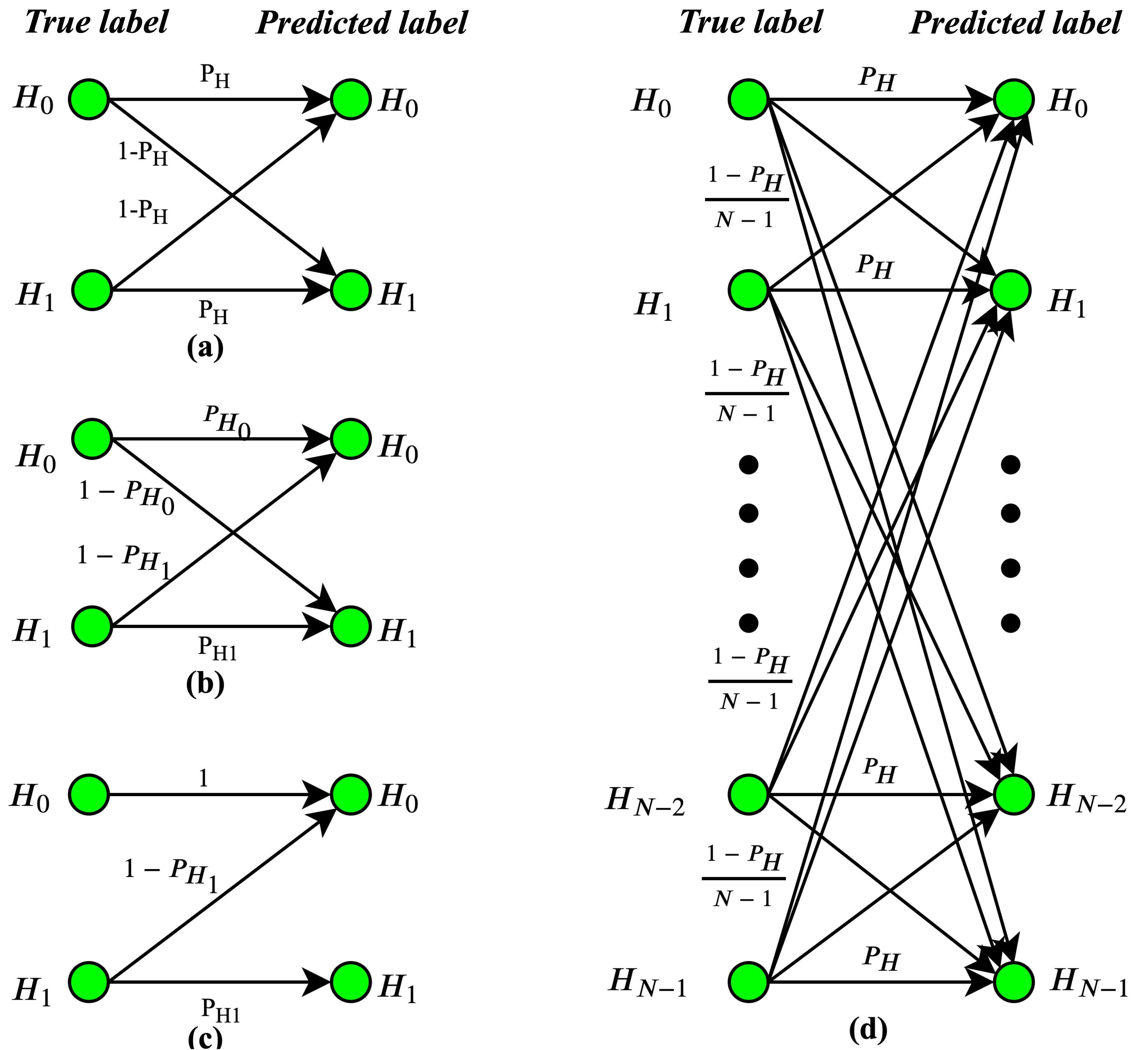}
    \caption{Discrete memoryless channels as potential models for quantifying the label bits that teacher can provide for the student: (a) Binary Symmetric Channel; (b) Binary Asymmetric Channel; (c) Z-Channel; (d) Q-ary Symmetric Channel, where N indicates the cardinality of input alphabet.}
    \label{fig1}
\end{figure}

The relative probabilities of incorrect class prediction (i.e., dark knowledge) can reveal a lot about the teacher generalization tendencies of the teacher. Soft targets probability of the teacher can extract and use the dark knowledge in conventional Knowledge Distillation (KD) \cite{kd}. As long as we distill the teacher's knowledge using soft logits at a high temperature, the amount of information how the teacher generalizes is linear in the number of classes \cite{skd}. When datasets contain many classes, knowledge transfer from teacher to student is typically successful, as the teacher has more relevant information about the function being taught \cite{skd}. Meanwhile, in classification tasks with a few classes or binary detection problems, the amount of information available to the student about the teacher's generalizability is restricted, thus limiting the utility of the KD approach. To address this problem, we can leverage hidden subclass knowledge, the knowledge of available subclasses that is not captured in the teacher's class logits. We propose the so-called Subclass Knowledge Distillation (SKD) framework to let us pass information from the teacher to the student via the soft subclass targets of the teacher (i.e., relative intra-class probabilities), which contain both hidden subclass and class knowledge.



In this paper, we introduce new tools and concepts from the field of information theory that can be applied to the machine learning community. From the perspective of Shannon's information theory, we investigate how well subclass knowledge can increase the information about the teacher's generalization through the SKD framework. To begin, we derive a closed-form expression for the number of label information bits that a teacher can provide to a student in a balanced dataset with the same number of training samples per class. Then, we establish an upper bound on the label information bits transmitted by the teacher network for real-world datasets with a biased distribution of training samples across classes or subclasses. It is worth nothing that our theoretical analysis allows us to figure out what is happening before we train the student. In other words, the teacher is evaluated first using our mathematical theorems, and then the potential teacher can train the student using SKD provided that it has sufficient knowledge to transfer. This is a minimalistic approach, which can help us avoid additional costly experimental resources. The proposed framework is evaluated on the Minimalistic HISTopathology (MHIST) dataset \cite{mhist} using a clinically binary classification task for colorectal polyps. We show that experimentally, the SKD framework utilizes more information about the generalization of the teacher than conventional KD does due to the learning of more fine-grained features, thereby improving student performance.  
To conclude, we summarize our contributions as follows:
\begin{itemize}
    \item We propose Subclass Knowledge Distillation (SKD), a novel framework to efficiently distill subclass knowledge into the student network and further boost its performance.
    
    \item We analyze how much extra information the student can learn about the teacher's generalization through the SKD framework in comparison to conventional KD.  

    \item We conduct an experimental study on the MHIST dataset, a clinically-important binary dataset, to evaluate the performance of the SKD framework and the motivation for transferring the subclass knowledge. Our experimental results demonstrate that the learned subclass factorization is useful for distillation and increases lightweight student performance due to the learning of more fine-grained features.
\end{itemize}

\section{Related Work} \label{sec2}

\textbf{Knowledge Distillation in Classification Tasks.} 
Transferring knowledge from one model to another is a research topic that has obtained noteworthy attention during recent years. Ba and Caruana \cite{caruana} trained a single and small neural network to imitate the logits of a large and complex neural network. Then, Hinton {\em et al.} \cite{kd} introduced KD and dark knowledge to claim that the deeper teacher model can successfully distill its knowledge into the smaller student neural network by matching their soft targets (softmax distributions).

Nowadays, a lot of successive papers have been written to propose different techniques to KD for model compression purposes. Romero {\em et al.} \cite{fitnet} distilled the feature representations of the teacher's intermediate layers to the student for improving the training stage of the student network. Transferring the attention maps \cite{attention, nst, tarvainen}, the inner products of intermediate activation maps \cite{innerproduct}, and relational knowledge between training samples \cite{relationalkd, tung, peng, liu} are some other methods to promote the distillation process from one model to another. However, these approaches ignored the possibility of available subclass knowledge within the classes and therefore did not take advantage of hidden subclass knowledge to improve student performance. By contrast, in this study, we use subclass knowledge to enhance the generalization ability of the teacher network.

\textbf{Subclass Knowledge Distillation.} 
The distillation of knowledge can be improved by increasing the amount of information that the teacher can transfer to the student. Müller {\em et al.} \cite{skd} compelled the teacher to create semantically meaningful subclasses for each class during its training phase with auxiliary contrastive loss. The student is then trained to mimic the invented teacher's subclasses predictions (probabilities). When the number of training samples per class is the same, they measured the number of bits of label information about how the teacher generalizes through subclass distillation in binary classification tasks. It should be noted that they only consider the cases where the number of subclasses in each class is equal, which is not always the case in real-world datasets. Maria {\em et al.} \cite{oskd, oskd2} proposed Online Subclass Knowledge Distillation (OSKD) as a method for estimating a set of subgroups and then training the lightweight model using a single-stage self distillation approach. These subgroups are estimated based on the different numbers of nearest neighbors in each sample. Although they show experimentally that KD can be improved by revealing the estimated subclass knowledge, they provide no mathematical justification for their work. Unlike previous methods for subclass distillation in the literature, we use known and available subclasses within each class. Our research aims to bridge the gap between mathematical and semantic explanation in a framework for subclass knowledge distillation.

\textbf{Histopathological Classification.}
In histopathology, image analysis is used to evaluate the characteristics of biopsies or manually inspected specimens under a pathologist's microscope. In recent years, deep learning has increased interest in using neural networks to analyze histopathology images, to the point where state-of-the-art Convolutional Neural Networks (CNNs) can perform at the level of pathologists in a variety of tasks \cite{HEKLER201979, prostatecancer}. Colorectal cancer is one of the most common types of cancer in the United States in $2021$ \cite{siegel2021cancer}. As a result, classification of colorectal polyps (small aggregates of cells that form on the surface of the colon, which can evolve into colonic cancer if left untreated) is a critical pathology task. In this paper, we evaluate our proposed framework in its clinically important application for gastrointestinal pathology. 

\section{Subclass Knowledge Distillation (SKD)} \label{sec3}

As mentioned previously, Subclass Knowledge Distillation is expected to make use of subclass knowledge when performing classification tasks with a small number of classes. In the following, we will elaborate on the details of the SKD framework in a teacher-student context where subclasses are known and available.

To begin, the teacher network trains on $S$ subclass labels and computes logits and subclass probabilities using the pre-softmax and softmax functions, $f_{T}(.)$ and $\sigma(.)$, respectively. In other words, the teacher network is trained using the ground-truth supervision of subclass labels by minimizing the cross-entropy loss (CE) associated with subclass probabilities that add to one:
\begin{flalign} \label{eq1}
 \mathcal{L}_{teacher} = \sum_{x_{i}\in \mathcal{X}}CE(\sigma(f_{T}(x_{i})), y_{i})
\end{flalign}
where $y_{i}\in \{0, 1\}^{S}$ indicates the one-hot encoded subclass label corresponding to the training sample $x_{i}$ of sample space $\mathcal{X}$.

The SKD framework, like conventional KD, is a process where the student is trained to mimic the teacher's behavior. However, instead of utilizing $C$ original classes, the student learns to match the teacher's output with $S$ subclasses that are always greater in number. The student creates $S$ output probabilities for each training sample $x_{i}$, resulting in the following SKD loss:
\begin{flalign} \label{eq2}
 \mathcal{L}_{skd} = \sum_{x_{i}\in \mathcal{X}} KL(\sigma(\frac{f_{T}(x_{i})}{\tau}), \sigma(\frac{f_{S}(x_{i})}{\tau}))
\end{flalign}
where KL denotes the Kullback-Leibler divergence and $f_{T}(.)$ and $f_{S}(.)$ are the teacher and student pre-softmax functions, respectively. The temperature hyperparameter, $tau$, is used to generate soft predictions while controlling the entropy of the output distribution.As the objective function for training the student network, we use a linear combination of the SKD loss $\mathcal{L}_{skd}$ and the standard cross entropy loss:
\begin{flalign} \label{eq3}
 \mathcal{L}_{student} = \lambda \mathcal{L}_{ce} + (1-\lambda) \mathcal{L}_{skd}
\end{flalign}
where $\mathcal{L}_{ce} = \sum_{x_{i}\in \mathcal{X}}CE(\sigma(f_{S}(x_{i})), y_{i})$ is the cross-entropy loss and $\lambda \in [0, 1]$ is a task balance hyperparameter. Following supervision of subclass labels to train the teacher and student networks, class output predictions can be determined simply by adding the probabilities of all subclasses within the class. Note that while we trained the teacher and student on subclass labels, they are evaluated on the class labels.     

In classification tasks with a few number of classes, the amount of knowledge the student learns about the teacher's generalization is limited. Thus, in this case, the SKD framework can help the student improve performance by leveraging hidden subclass knowledge, the additional knowledge of known subclasses within each class. This knowledge of subclasses is useful in the SKD framework because subclass labels allow the teacher to learn more features than class labels. When evaluating on class labels, these fine-grained subclass labels can help the teacher generalize better. Assume that all samples in the training set from class $i$ have the unique feature $f_{1}$ that no other class has. Due to the network's tendency to learn only the most discriminative features \cite{bilen2016weakly}, the teacher learns the feature $f_{1}$ to predict the class $i$ as long as we train the teacher with class labels. Then, in the test set, If the new image of class $i$ lacks the feature $f_{1}$, the teacher will predict the wrong class. Subclass training enables the teacher to learn more features necessary for predicting fine-grained subclasses. These additional features can improve the teacher performance in class-level classification tasks (the teacher can correctly predict a new sample of class $i$ if it has the other features trained by subclass labels, even if the sample lacks the feature $f_{1}$).

To demonstrate the effectiveness of our framework, we calculate the number of label bits that the teacher can provide to the student using different types of discrete memoryless channels. The information theory channel is a system whose output is probabilistically dependent on its input \cite{cover}. Every channel is defined by an input alphabet, an output alphabet, and a description of how the output depends on the input. In this paper, the true label space and the predicted label space are the input and the output alphabets of our channel, namely $\mathcal{A}$ and $\mathcal{\hat{A}}$, respectively. Similar to the channel transition matrix in information theory, the normalized confusion matrix on the training set illustrates the relationship between the predicted and true labels of the teacher network. Furthermore, and most importantly, the information capacity of each channel indicates the amount of information it transmits, which is equivalent to the information label bits that the teacher can provide to the student in our study. It should be mentioned that all channels used in the paper are memoryless, as each predicted label is influenced only by the corresponding true label, not by earlier true or predicted labels.

In the following theorem, we measured how many label information bits the teacher can transfer to the student when each class contains the same number of training samples as well as an equal number of subclasses. This is the case for datasets such as MNIST, CIFAR100, and ImageNet dog vs. cat \cite{chen2018understanding}.
\begin{theorem} \label{th1}
Suppose that the dataset is balanced, i.e., the number of training samples per class is the same and each class has $N_{S}$ subclasses. Let the teacher predict each subclass and class correctly with a probability of $P_{S}$ and $P_{C}$. As long as their remaining error probabilities are identically distributed among the remaining subclasses and classes, the number of label information bits the teacher can transfer is characterized by 
\begin{flalign} \label{eq4}
 &[\log N_{C} + P_{C}\log P_{C} + (1-P_{C}) \log  \frac{1-P_{C}}{N_{C}-1}] +  \nonumber \\
 & [\log N_{S} + P_{S}\log P_{S} + (1-P_{S}) \log  \frac{1-P_{S}}{N_{S}-1}]
\end{flalign}
where the first and second parts are derived from the class and subclass labels, respectively. $N_{C}$ represent the total number of classes. (In this paper, all information quantities are represented in bits, and the $\log$ function is to base $2$.) 
\end{theorem}

\begin{proof}
Since the remaining error probabilities for each class (subclass, resp.) are distributed uniformly over the remaining $(N_{C}-1)$ classes ($(N_{S}-1)$ subclasses, resp.), the normalized confusion matrix for class (subclass, resp.) classification will follow the structural pattern of the transition matrix for ``\textbf{Q-ary Symmetric Channel}" (Figure \ref{fig1}(d): $N = N_{C}$ and $P_{H} = P_{C}$ for class labels and $N = N_{S}$ and $P_{H} = P_{S}$ for subclass labels). While the number of training samples per class is the same, the knowledge that the teacher can provide via the class label is equal to the capacity of Q-ary Symmetric channel \cite{cover}. This capacity is achieved by a uniform input distribution and is given by:
\begin{flalign} \nonumber
& \log N_{C} - H(P_{C}, \frac{1-P_{C}}{N_{C}-1}, ..., \frac{1-P_{C}}{N_{C}-1}) =  \nonumber \\
 & \log N_{C} + P_{C} \log P_{C} + (1-P_{C}) \log \frac{1-P_{C}}{N_{C}-1} \nonumber
\end{flalign}
where $H(x_{1},x_{2}, ..., x_{n}) = -\sum_{i=1}^{n}x_{i}\log x_{i}$ denotes entropy function. \\
The capacity of the Q-ary Symmetric Channel is also used to calculate how much information $N_S$ subclasses of the class $i$ can transfer to the student. Provided that each class contains the same number of training samples, the following weighted average can be used to determine the average number of subclass label bits per training sample.
\begin{flalign} \nonumber
 \sum_{i = 1}^{N_{C}}
& \frac{1}{N_C}[\log N_{S} - H(P_{S}, \frac{1-P_{S}}{N_{S}-1}, ..., \frac{1-P_{S}}{N_{S}-1})]. \mathbb{1}_{i > 0} = \nonumber \\
& \frac{N_{C}}{N_{C}}[\log N_{S} - H(P_{S}, \frac{1-P_{S}}{N_{S}-1}, ..., \frac{1-P_{S}}{N_{S}-1})] =  \nonumber \\
 & \log N_{S} + P_{S} \log P_{S} + (1-P_{S}) \log \frac{1-P_{S}}{N_{S}-1} \nonumber
\end{flalign}
where $\frac{1}{N_{C}}$ is the weight applied to the subclass label bits of class $i$ and $\mathbb{1}$ denotes the indicator function. Finally, the summation of class and subclass label information completes the proof of Theorem \ref{th1}.
\end{proof}

Müller {\em et al.}'s \cite{skd} analytical measurement on label bits is a special case of Theorem \ref{th1}. When the teacher network learns the subclasses perfectly, it provides $(\log{N_{C}}+\log{N_{S}}) = \log{(N_{C}N_{S})}$ label bits per training sample, which is a meaningful expression.

Although Müller {\em et al.} have done some analytical work on subclass distillation, their work is limited to balanced datasets such as balanced CIFAR-10 \cite{cifar10} and MNIST \cite{mnist}. In real-world datasets such as those from the medical field \cite{Dua:2019, bhattacharjee2001classification, esteva2017dermatologist}, each class may contain a different number of subclasses, and each subclass has multiple training instances. In the rest of this section, we are considering a particular type of classification problem. Within the problem of data-driven classification, there is a case where basically it is extremely important, which is the so-called detection case. The detection task is primarily a binary classification of the hypothesis under consideration, with the result typically being either a null hypothesis $H_{0}$ or an alternative hypothesis $H_{1}$. The detection problem is practically important because, for example, this is the case when someone tries to detect whether a person has cancer or not. In cancer diagnosis tasks, alternative hypothesis $H_{1}$ may have $N_{H}$ subclasses in which each of them can express different types or organs of cancer disease \cite{oakden2020hidden, mlynarski2019deep} (Fig. \ref{fig2}). Furthermore, it is fair to say that the majority of training samples are identified as normal class \cite{imbalance}, resulting in a biased dataset. For our binary detection task, we establish an upper bound on the number of label bits the teacher can transfer to help the student generalize better by Theorem \ref{th2}. The detailed analysis of the general case, a multiclass classification task with a different number of subclasses in each class, can be found in the supplementary material.
  
\begin{figure} 
    \centering
    \includegraphics[width= 0.83\linewidth]{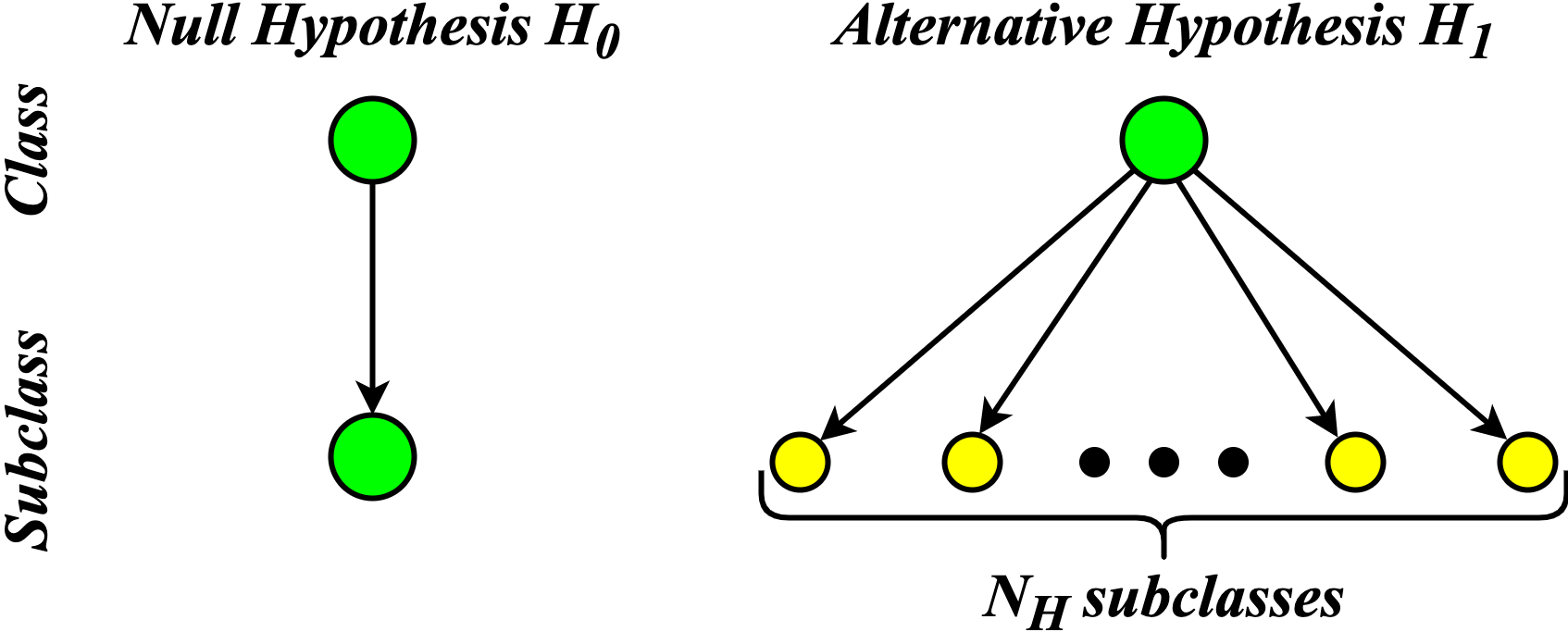}
    \caption{Class hierarchy of our binary detection problem.}
    \label{fig2}
\end{figure}

\begin{theorem} \label{th2}
Let the teacher network predict the null and alternative hypothesis correctly with a probability of $P_{H_{0}}$ and $P_{H_{1}}$, respectively. In the case the teacher predicts each subclass of alternative hypothesis properly with a probability of $P_{H_{11}}$ and the remaining errors are equally distributed throughout the remaining $(N_{H}-1)$ subclasses, the average number of label bits per training sample that the teacher can provide is bounded above by
\begin{flalign} \nonumber 
& [\log (1+2^{K(P_{H_{0}}, P_{H_{1}})}) -P_{H_{0}}K(P_{H_{0}}, P_{H_{1}}) - H_{b}(P_{H_{0}})] + \nonumber\\
 & [\alpha^{*} (\log  N_{H}+P_{H_{11}}\log P_{H_{11}} +  (1-P_{H_{11}})\log  \frac{1-P_{H_{11}}}{N_{H}-1})]\nonumber
\end{flalign}
where $H_{b}(x) = -x\log{x} - (1-x)\log(1-x)$ denotes a binary entropy function, $K(P_{H_{0}}, P_{H_{1}}) = \frac{H_{b}(P_{H_{1}}) - H_{b}(P_{H_{0}})}{P_{H_{0}}+P_{H_{1}}-1}$, and $\alpha^{*}$ is equal to
\begin{flalign} \nonumber
 &\frac{1}{(P_{H_{0}}+P_{H_{1}}-1)(2^{K(P_{H_{0}}, P_{H_{1}})}+1)} - \frac{1-P_{H_{0}}}{P_{H_{0}}+P_{H_{1}}-1}.
\end{flalign}
\end{theorem}

\begin{proof}
Inasmuch as the task is binary detection and the teacher predicts both classes with different probabilities in general, the normalized confusion matrix for the class classification task will follow the structural pattern of the channel transition matrix for \textbf{Binary Asymmetric Channel (BAC)} (Fig. \ref{fig1}(b)). Therefore, the information capacity of BAC tells us the maximum number of label bits that the teacher can convey to the student via class labels. Let $Y$ and $\hat{Y}$ be random variables taking values in $\mathcal{A}$ and $\mathcal{\hat{A}}$, respectively. Without loss of generality, we assume that $P_{H_{1}} \leq P_{H_{0}}$ and $\alpha$ denotes the probability of event that the training sample belongs to the alternative hypothesis. The capacity of BAC is then given by      
\begin{flalign} \label{eq5}
C_{BAC} = \max_{\alpha}& \hspace{4pt} I(Y;\hat{Y}) \overset{(a)}= \max_{\alpha} \hspace{4pt} (H(\hat{Y})-H(\hat{Y}|Y))   \nonumber \\
  \overset{(b)}= \max_{\alpha} & [H_{b}(\alpha P_{H_{1}} + (1-\alpha)(1-P_{H_{0}})) - \nonumber \\ 
  & (1-\alpha)H_{b}(P_{H_{0}}) - \alpha H_{b}(P_{H_{1}})] 
\end{flalign}
where $H(\hat{Y}|Y)$ denotes the conditional entropy of $\hat{Y}$ given $Y$ and $(a)$ and $(b)$ are followed from the definition of mutual information $I(Y;\hat{Y})$ and the mutual information corresponding to BAC, respectively. To achieve the optimal point, we calculate the derivative of the cost function with respect to $\alpha$ and, after some simplifications, obtain
\begin{flalign} \nonumber
 \alpha^{*} = \frac{1}{(P_{H_{0}}+P_{H_{1}}-1)} [\frac{1}{2^{K(P_{H_{0}}, P_{H_{1}})}+1}-(1-P_{H_{0}})]
\end{flalign}
where $K(P_{H_{0}}, P_{H_{1}}) = \frac{H_{b}(P_{H_{1}})-H_{b}(P_{H_{0}})}{P_{H_{0}}+P_{H_{1}}-1}$. The capacity of BAC is then calculated by substituting $\alpha = \alpha^{*}$ into the cost function of Equation \ref{eq5}, as shown below.
\begin{flalign}  \nonumber
\log (1+2^{K(P_{H_{0}}, P_{H_{1}})})-P_{H_{0}}K(P_{H_{0}}, P_{H_{1}}) - H_{b}(P_{H_{0}}) 
\end{flalign}
Unless the relative frequencies of training samples over classes match the capacity-achieving distributions, the capacity of BAC will be an upper bound on the number of label bits that the teacher can provide via class labels. In other words, if the relative frequency of training samples over alternative hypothesis converges to $\alpha^{*}$, the upper bound will be tight and converge to the real label bits. \\
In parallel with Theorem \ref{th1}, Q-ary Symmetric Channel could be a suitable model to analyze the subclass label bits since the normalized confusion matrix for the subclass classification task follows the structural pattern of the channel transition matrix for Q-ary Symmetric Channel (Fig. \ref{fig1}(d): $N=N_{H}$ and $P_{H} = P_{H_{11}}$). Q-ary symmetric channel capacity, in conjunction with proof of Theorem \ref{th1}, gives us the desired upper bound on the number of subclass label bits provided by the teacher to the student.
\begin{flalign} \nonumber
[(1-\alpha^{*})\times 0] + [\alpha^{*}(\log &N_{H}+{P_{H_{11}}}\log P_{H_{11}} \nonumber \\+ (1-P&_{H_{11}})\log \frac{1-P_{H_{11}}}{N_{H}-1})] = \nonumber \\ 
\alpha^{*}(\log N_{H}+{H_{11}}\log P_{H_{11}}&+(1-P_{H_{11}})\log \frac{1-P_{H_{11}}}{N_{H}-1})] \nonumber
\end{flalign}
where $\alpha^{*}$ ($1-\alpha^{*}$, resp.) denotes the relative frequency of training samples over alternative hypothesis (null hypothesis, resp.) when the upper bound on class label bits is tight. It is important to mention that there is one subclass in the normal class, which is the class itself. At the end, the summation of upper bounds on class and subclass label bits completes the proof of Theorem \ref{th2}.   
\end{proof}

In medical applications, such as cancer detection, the prediction of the normal class is easier than the abnormal one \cite{karabatak2015new, pawar2013breast}. Suppose that the teacher can, ideally, predict the null hypothesis. In this case, we can analyze the number of class and subclass label bits using \textbf{Z-Channel} (Fig. \ref{fig1}(c)), a particular type of BAC, and Q-ary Symmetric Channel, respectively. The following corollary illustrates this point.
\begin{corollary} \label{cor1}
Suppose that the Null Hypothesis $H_{0}$ is predicted ideally $(i.e., P_{H_{0}} = 1)$, and $P_{H_{1}}$ denotes the prediction probability of the alternative hypothesis. If all subclass conditions of Theorem $\ref{th2}$ are satisfied, the average number of label bits per training sample is bounded above by
\begin{flalign} \label{eq11}
[\log (1+2^{-K(P_{H_{1}})})] + &[\alpha^{*}(\log N_{H}+P_{H_{11}}\log P_{H_{11}}+\nonumber\\
   &(1-P_{H_{11}})\log \frac{1-P_{H_{11}}}{N_{H}-1})]
\end{flalign} \label{eq13}
where $K( P_{H_{1}}) = \frac{H_{b}(P_{H_{1}})}{P_{H_{1}}}$ and 
\begin{flalign} \nonumber
 \alpha^{*} = \frac{1}{P_{H_{1}}(2^{K(P_{H_{1}})}+1)} \in (0.3768, 0.5].
\end{flalign}
\end{corollary}
\begin{proof}
Because Z-Channel is a special case of BAC, the proof is identical to the proof of Theorem \ref{th2}. The upper bound on the subclass label bits follows exactly Theorem \ref{th2}'s proof line; however, the number of class label bits will be bounded above by 
\begin{flalign}
\log(1+2^{K(P_{H_{1}})}) - K(P_{H_{1}}) = \log(1+2^{-K(P_{H_{1}})}). \nonumber
\end{flalign}
This completes the proof of Corollary \ref{cor1}.
\end{proof}

Precise examination of Theorems \ref{th1} and \ref{th2} reveals that we need a condition under which the normalized confusion matrix for subclass and/or class classification follows the structural pattern of Q-ary Symmetric Channel's transition matrix. To generalize this condition, we can consider two structural patterns for the transition matrix: \textbf{(I) Strong Symmetric Channel}: each row (column, resp.) is a permutation of the other rows (columns, resp.); \textbf{(II) Weakly Symmetric Channel}: each row is a permutation of the other rows, and all the column sums are equal. Then we have the following remark.
\begin{remark} \label{rem1}
Let the normalized confusion matrix for subclass and/or class classification tasks follow either type (I) or (II). Then, the capacity of a strong symmetric channel or weakly symmetric channel provides us the maximum number of label bits the teacher can transfer to the student via the class and/or subclass labels. Thus, all the preceding results of Theorems \ref{th1} and \ref{th2} will hold if we substitute the following capacity with the capacity of Q-ary Symmetric Channel.
\begin{flalign} \nonumber
 C = \log (N) - H (\text{row of normalized confusion matrix})
\end{flalign}
where $N$ denotes the number of classes or subclasses. Note that the capacity is achieved by a uniform distribution on the input alphabet. 
\end{remark}
In this paper, we assume that the teacher provides a noisy version of one-hot encoded class or subclass labels to the student network. We can also gain from soft information as well as hard information if we utilize a higher temperature in knowledge distillation, which can increase the number of label bits per training sample. 

\section{Experimental Setup} \label{sec4}
In this section, we describe the experimental setup that will be used throughout the study. Our objective is to compress a large-scale teacher with high accuracy into a smaller student that is more appropriate for deployment. To this aim, we rely on the distillation of subclass knowledge through SKD.

\subsection{Minimalist HIStopathology (MHIST) dataset}
In this paper, we focus on the clinically-important classification problem between Hyperplastic Polyps (HPs) and Sessile Serrated Adenomas (SSAs) \cite{ssa, ssa_lesion, hp} on MHIST dataset \cite{mhist}. HPs are generally benign, but SSAs are precancerous lesions that, if left untreated, might progress to malignancy and require more frequent follow-up exams \cite{colonpolyps}. Pathologically, HPs have superficial serrations in the upper parts of the crypt, whereas in SSAs, serrations extend deeper into the crypt and the crypts are broad-based and may have a boot shape \cite{ssa2} (Fig. \ref{fig3}). In the annotation phase of MHIST dataset, seven practicing board-certified gastrointestinal pathologists separately and independently classified each of the $3,152$ images as either HP or SSA \cite{mhist}. The gold standard label was then allocated to each image on the basis of the majority vote among the seven labels, a common choice in literature \cite{breastmr, colorectaldeep}. The training set contains $2175$ examples, while the test set has $977$ images ($224 \times 224$ pixels). In addition, we use $5$-fold cross validation to tune the hyperparameters. 

\subsection{Subclass classification tasks}
In the MHIST dataset, each class can be partitioned into $4$ subgroups according to the discrete level of difficulty, which is determined by image-level annotator agreement: (I) very easy to predict ($7/7$ annotator agreement), (II) easy to predict ($6/7$ annotator agreement), (III) hard to predict ($5/7$ annotator agreement), and (IV) very hard to predict ($7/7$ annotator agreement). These clinician-provided annotations are used to define subclasses based on the annotation label's variability in a curriculum style of learning \cite{wei2021learn}. Then, we take the following classification tasks:
\begin{enumerate}
     \item \textbf{ClassLevel-}$\mathbf{11}$\textbf{:} We do not have subclasses. The task is class binary classification.
     
     \item \textbf{SubclassLevel-}$\mathbf{21}$\textbf{:} SSA has $2$ subclasses: (I) very easy and easy to predict, (II) hard and very hard to predict. HP has a one subclass, which is the class itself.
     
     \item \textbf{SubclassLevel-}$\mathbf{41}$\textbf{:} Each level of difficulty is a subclass for SSA. HP has a one subclass, which is the class itself.
     
     \item \textbf{SubclassLevel-}$\mathbf{22}$\textbf{:} Both SSA and HP has $2$ subclasses: (I) very easy and easy to predict, (II) hard and very hard to predict.
     
     \item \textbf{SubclassLevel-}$\mathbf{12}$\textbf{:} SSA has a single subclass, which is the class itself. HP has $2$ subclasses: (I) very easy and easy to predict, (II) hard and very hard to predict.
     
     \item \textbf{SubclassLevel-}$\mathbf{14}$\textbf{:} SSA has a one subclass, which is the class itself. Each level of difficulty is a subclass for HP.
\end{enumerate}

\begin{figure} 
    \centering
    \includegraphics[width= 1\linewidth]{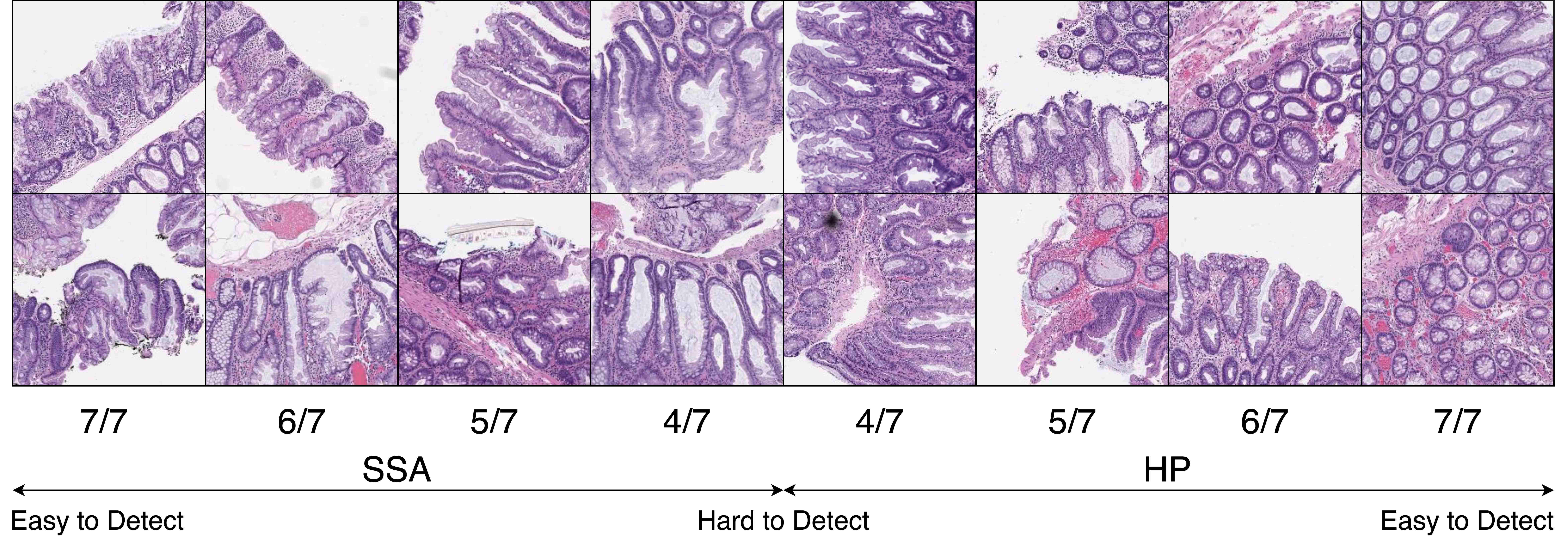}
    \caption{Example images for MHIST dataset \cite{mhist}.}
    \label{fig3}
\end{figure}

\subsection{Teacher and student models}  
We trained the ResNet50V2 \cite{resnet} network to be used as a teacher for our classification task. Specifically, we trained the teacher for $50$ epochs using data augmentation with Adam optimizer \cite{adam}, batch size of $32$, learning rate of $1\times10^{-4}$, and weight decay of $5\times10^{-4}$. Dropout is also implemented with probability $0.2$ to the final softmax matrix. For the student's models, we used NASNetMobile \cite{nasnetmobile} network with data augmentation, a minibatch size of $32$ and Adam optimizer \cite{adam}. For the baseline student that is trained from scratch without class or subclass knowledge distillation, we ran a grid search over a number of epochs ($20, \mathbf{30}, 50, 100, 150$), learning rate ($0.1$, $0.01$, $\mathbf{0.001}$, $0.0001$), and weight decay ($\mathbf{0.0005}$, $0.005$, $0.05$). Values with the highest accuracy are highlighted in bold. In addition, for all student models that distillation is used, we tune the distillation hyperparameters, such as temperature $\tau$ ($1$, $2$, $\mathbf{5}$ \textbf{(tasks} $\mathbf{2}$, $\mathbf{3}$, $\mathbf{4}$, $\mathbf{5}$, \textbf{and} $\mathbf{6}$\textbf{)}, $16$, $32$, $64$, $\mathbf{128}$ \textbf{(task} $\mathbf{1}$\textbf{)} and task balance $\lambda$ ($0.0$, $0.25$, $\mathbf{0.45}$ \textbf{(tasks} $\mathbf{1}$, $\mathbf{2}$, $\mathbf{3}$, $\mathbf{5}$, \textbf{and} $\mathbf{6}$\textbf{)}, $\mathbf{0.75}$ \textbf{(task} $\mathbf{4}$\textbf{)}). All the models in this paper are trained on NVIDIA Tesla V100-SXM2-16GB \cite{teslav100} using the TensorFlow framework.

\section{Experimental Results} \label{sec5}
In this section, we conduct experiments on the principles described in the preceding sections using the MHIST dataset. The MHIST dataset is a skewed dataset in which the HP has more samples than the SSA ($2162$ samples for HP, $990$ samples for SSA). When the dataset is unbalanced, it is important to find an equal balance between precision and recall. So, we use the F-1 score to compare the networks' performance in our experiment. In all experiments, teachers and students were trained on subclass labels but evaluated on class labels. We started by training the teacher network on MHIST for the CL-$11$ task and obtained an F1-score of $85.57\%$ (averaged over $60$ runs, as all the results in this section). We trained the student network without knowledge distillation in the task with no subclasses as a baseline and observed a gap of $2.66\%$ between the teacher and the student performance. Then, we investigate how the distillation of knowledge can help the student perform better. We obtained an $0.65\%$ increase in F1-score for the student trained with conventional KD, compared to the baseline student. Following that, we distilled subclass knowledge from teachers that had been trained for the subclass classification tasks in the previous section and find a class F1-score improvement compared to the baseline student (Table \ref{table1}). Specifically, the student learns with SKD in SL-$12$ task, achieves an F1-score of $85.05\%$, an improvement of $2.14\%$ and $1.49\%$ over the student that trains without and with conventional KD, respectively. These results show that the SKD framework can compress a large-scale teacher into a smaller and less computational complexity student without severely sacrificing its performance. To be more precise, we measure the computational cost of teacher and student networks trained with SKD in the SL-$12$ task using the number of multiply-adds (FLOPs) as described in \cite{xie2017aggregated}. As shown in Table \ref{table2}, the computational complexity of the student network is $6x$ less than that of the teacher, while its inference time  is roughly equal to the teacher's inference time.

\begin{table}[htb] 
\centering
\begin{tabular}{c c c}
\hline
{Task} & {Method} & {Binary Class F1-score(\%)}\\
\hline

\multirow{3}{*}{CL-$11$} & {Teacher} & $85.57\pm0.81$\\ 
& {Student (baseline)} & $\mathbf{82.91\pm1.02}$\\
& {Student + KD} & $83.56\pm1.64$\\
\hline

\multirow{3}{*}{SL-$21$} & {Teacher} & $85.78\pm0.99$\\ 
& {Student} & $83.47\pm1.84$\\
& {Student + SKD} & $84.52\pm1.54$\\
\hline

\multirow{3}{*}{SL-$41$} & {Teacher} & $85.56\pm0.93$\\ 
& {Student} & $83.64\pm1.55$\\
& {Student + SKD} & $84.32\pm1.18$\\
\hline

\multirow{3}{*}{SL-$22$} & {Teacher} & $85.75\pm0.94$\\ 
& {Student} & $83.89\pm1.48$\\
& {Student + SKD} & $84.94\pm1.34$\\
\hline

\multirow{3}{*}{SL-$12$} & {Teacher} & $85.97\pm0.87$\\ 
& {Student} & $84.16\pm1.75$\\
& {Student + SKD} & $\mathbf{85.05\pm1.48}$\\
\hline

\multirow{3}{*}{SL-$14$} & {Teacher} & $85.40\pm0.99$\\ 
& {Student} & $83.42\pm1.46$\\
& {Student + SKD} & $84.28\pm1.58$\\
\hline
\end{tabular}
\caption{Results of test F-1 score in different tasks. The baseline corresponds to training the student without distillation. The distillation results correspond to training the student to match its teacher’s class and subclass predictions with KD and SKD, respectively. (SL: SubclassLevel, CL: ClassLevel)}
\label{table1}
\end{table}

\begin{table}[htb]
\centering
\begin{tabular}{c c c c}
\hline
{Model} & {FLOPs} & {Inference time (ms)} & Parameters\\
\hline
{Teacher} & $6.970$G & $5.29$ & $20.57$M\\
{Student} & $\mathbf{1.136}$\textbf{G} &  $6.81$ & $\mathbf{2.21}$ \textbf{M}\\
\hline
\end{tabular}
\caption{Results of computational cost (G-FLOPs), interference time, and the trainable parameters for the teacher and the student networks trained in the SL-$12$ task.}
\label{table2}
\end{table}

\begin{table}[htb]
\centering
\begin{tabular}{c c}
\hline
{Task} &  {Total label information bits/sample}\\
\hline
{ClassLevel-$11$} & $0.8363$ \\
{SubclassLevel-$21$} & $1.1664$ \\
{SubclassLevel-$41$} & $1.1684$ \\
{SubclassLevel-$22$} & $1.2758$ \\
{SubclassLevel-$12$} & $\mathbf{1.3019}$ \\
{SubclassLevel-$14$} & $1.1556$ \\
\hline
\end{tabular}
\caption{The upper bound on the number of label bits per sample that the teacher can provide in different tasks. Total label bits is the summation of class and subclass label bits per sample.}
\label{table3}
\end{table}
We also measured the label bits that the teacher can transfer to the student in order to show how SKD can benefit from subclass knowledge to help the student perform better.  The results in Table \ref{table3} show that the student, trained on the SKD framework, can gain $0.4656$ extra label bits per sample from hidden subclass knowledge. This difference in the number of label bits explains the $2.14\%$ F1-score gap between the students trained with and without subclass distillation in the binary classification task. The details of measuring the label bits can be found in the supplementary material.

\begin{figure} 
    \centering
    \includegraphics[width=0.89\linewidth]{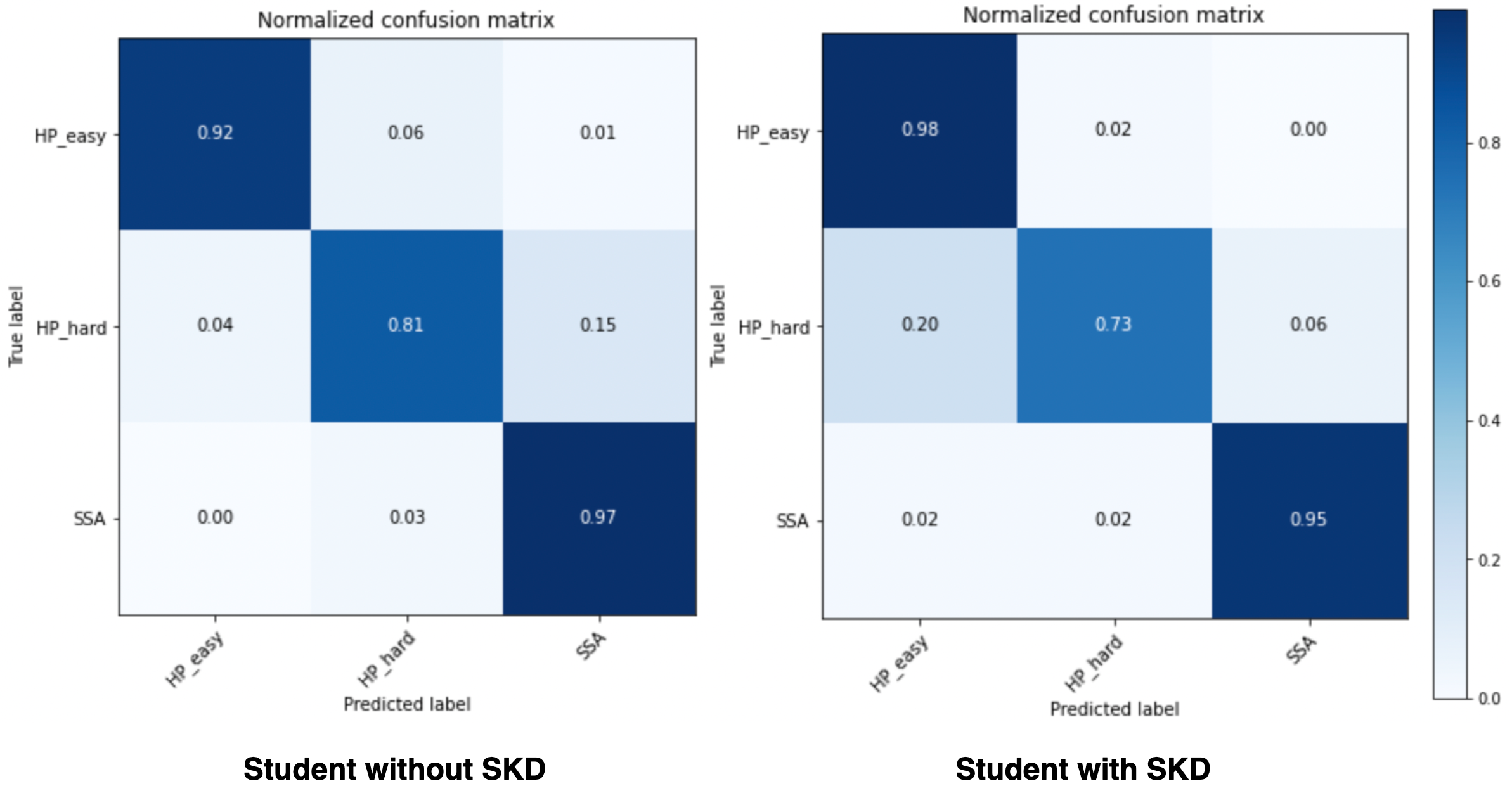}
    \caption{Normalized confusion matrix on the test set for student trained in the SL-$12$ task with and without SKD.}
    \label{fig4}
\end{figure}

In the SL-$21$ task, the HP-easy and HP-hard subclasses have similar fine-grained feature that the teacher can learn them and transfer their dark knowledge (high $\tau$) to the student via SKD. When the true label is HP-hard, this knowledge may affect on the student's probability values for incorrect subclasses. As shown in Figure \ref{fig4}, when the student is trained with subclass distillation, the HP-hard's error probability is more biased towards the HP-easy than the SSA. Therefore, this can boost the class- level performance.

\section{Conclusion and Future Works} \label{sec6}
In this paper, we propose Subclass Knowledge Distillation for classification where information on existing subclasses is available and taken into consideration. We show that we can improve the performance of the lightweight student by transferring hidden subclass knowledge, the additional meaningful information that helps the teacher to learn more fine-grained features. We also mathematically measure this extra knowledge using channel capacity concepts. Finally, the SKD was evaluated in the clinical binary classification and showed that it can benefit from subclass knowledge to boost student performance. Future works could be theoretically, such as investigating the proposed upper bound's tightness, or experimentally, like evaluating SKD on more datasets.
\bibliography{SubclassKD_AAAI22.bib}

\onecolumn
\def\x{{\mathbf x}}
\def\L{{\cal L}}

\begin{center}
\textbf{\LARGE Supplementary Materials}
\end{center}

\appendix

\section{Analyzing the Number of Label Information Bits per Training Sample (General Case)}

This subsection analyzes the total number of label information bits that the teacher can transfer to the student in a multiclass classification task with a different number of subclasses per class (Figure \ref{fig5}) using the Subclass Knowledge Distillation (SKD) framework. This general case is described in the following theorem.

\begin{theorem} \label{th3}
Suppose that the dataset has $N_{C}$ classes, and class $i$ contains known and available $N_{C_{i}}$ subclasses for $i$ in $\{1, 2, ..., N_{C}\}$. Let the teacher network predict each class $i$ (each subclass of class $i$, resp.) correctly with a probability of $P_{C}$ ($P_{C_{i}}$, resp.) during the training phase, while the remaining errors are distributed equally over the remaining $(N_{C}-1)$ classes ($(N_{C_{i}}-1)$ subclasses, resp.). Then, the teacher's ability to provide the student with label information is bounded above by
\begin{flalign} \label{eq7}
 (\log N_{C} + P_{C}\log P_{C} + (1-P_{C}) \log  \frac{1-P_{C}}{N_{C}-1}) +  \sum_{k = 1}^{N_{C}} \frac{\sum_{j = 1}^{N_{C_{k}}} N_{S_{kj}}}{\sum_{i = 1}^{N_{C}} \sum_{j = 1}^{N_{C_{i}}} N_{S_{ij}}} (log N_{C_{k}} + P_{C_{k}} log P_{C_{k}} + (1-P_{C_{k}}) log \frac{1-P_{C_{k}}}{N_{C_{k}}-1})
\end{flalign}
where $N_{S_{ij}}$ indicates the number of training sample for subclass $j$ of class $i$.  All information quantities are represented in bits, and the $\log$ function is to base $2$.
\end{theorem}

\begin{proof}
As long as the teacher predicts each class correctly with a probability of $P_{C}$ and the remaining error probabilities are distributed uniformly across the remaining classes, the class-level normalized confusion matrix for the training set will follow the structural pattern of the Q-ary Symmetric Channel's transition matrix. Therefore, the information capacity of the Q-ary Symmetric Channel indicates the maximum number of label bits that a teacher can convey to a student via class labels. This capacity is achieved by a uniform distribution across the class label space and is given by
\begin{flalign} \label{eq8}
\log N_{C} - H(P_{C}, \frac{1-P_{C}}{N_{C}-1}, \frac{1-P_{C}}{N_{C}-1}, \ldots, \frac{1-P_{C}}{N_{C}-1}) = \log N_{C} + P_{C} \log P_{C} + (1-P_{C}) \log \frac{1-P_{C}}{N_{C}-1} 
\end{flalign}
where $H(x_{1},x_{2}, ..., x_{n}) = -\sum_{i=1}^{n}x_{i}\log x_{i}$ denotes the entropy function. Unless the relative frequencies of training samples across classes match the capacity-achieving distributions, the Q-ary Symmetric Channel's capacity will be an upper bound on the number of label information bits that the teacher can provide using class labels. In other words, if the relative frequencies of training samples over class labels converge to a uniform distribution, the upper bound will be tight and will approach the real label bits.\\
In parallel with the class labels, the Q-ary Symmetric Channel could also be a suitable model to analyze the subclass label bits for a given class $i$, because the subclass-level normalized confusion matrix for the training set of class $i$ follows the structural pattern of the Q-ary Symmetric Channel's transition matrix. Given that each subclass has a different number of training samples, the following weighted average can be used to further establish an upper bound on the number of subclass label bits per sample that the teacher can provide.  
\begin{flalign} \label{eq9}
&\frac{\sum_{j = 1}^{N_{C_{1}}} N_{S_{1j}}}{\sum_{i = 1}^{N_{C}} \sum_{j = 1}^{N_{C_{i}}} N_{S_{ij}}} (\log N_{C_{1}} - H(P_{C_{1}}, \frac{1-P_{C_{1}}}{N_{C_{1}}-1}, \frac{1-P_{C_{1}}}{N_{C_{1}}-1}, \ldots, \frac{1-P_{C_{1}}}{N_{C_{1}}-1})+ \nonumber\\
&\frac{\sum_{j = 1}^{N_{C_{2}}} N_{S_{2j}}}{\sum_{i = 1}^{N_{C}} \sum_{j = 1}^{N_{C_{i}}} N_{S_{ij}}} (\log N_{C_{2}} - H(P_{C_{2}}, \frac{1-P_{C_{2}}}{N_{C_{2}}-1}, \frac{1-P_{C_{2}}}{N_{C_{2}}-1}, \ldots, \frac{1-P_{C_{2}}}{N_{C_{2}}-1}))+ \ldots + \nonumber\\
&\frac{\sum_{j = 1}^{N_{C_{N_{C}}}} N_{S_{N_{C}j}}}{\sum_{i = 1}^{N_{C}} \sum_{j = 1}^{N_{C_{i}}} N_{S_{ij}}} (\log N_{C_{N_{C}}} - H(P_{C_{N_{C}}}, \frac{1-P_{C_{N_{C}}}}{N_{C_{N_{C}}}-1}, \frac{1-P_{C_{N_{C}}}}{N_{C_{N_{C}}}-1}, \ldots, \frac{1-P_{C_{N_{C}}}}{N_{C_{N_{C}}}-1})) = \nonumber \\
&\sum_{k = 1}^{N_{C}} \frac{\sum_{j = 1}^{N_{C_{k}}} N_{S_{kj}}}{\sum_{i = 1}^{N_{C}} \sum_{j = 1}^{N_{C_{i}}} N_{S_{ij}}} (\log N_{C_{k}} - H(P_{C_{k}}, \frac{1-P_{C_{k}}}{N_{C_{k}}-1}, \frac{1-P_{C_{k}}}{N_{C_{k}}-1}, \ldots, \frac{1-P_{C_{2}}}{N_{C_{2}}-1})) = \nonumber \\
& \sum_{k = 1}^{N_{C}} \frac{\sum_{j = 1}^{N_{C_{k}}} N_{S_{kj}}}{\sum_{i = 1}^{N_{C}} \sum_{j = 1}^{N_{C_{i}}} N_{S_{ij}}} (\log N_{C_{k}} + P_{C_{k}} \log P_{C_{k}} + (1-P_{C_{k}}) \log \frac{1-P_{C_{k}}}{N_{C_{k}}-1})
\end{flalign}
where $\frac{\sum_{j = 1}^{N_{C_{i}}} N_{S_{ij}}}{\sum_{i = 1}^{N_{C}} \sum_{j = 1}^{N_{C_{i}}} N_{S_{ij}}}$ is the weight applied to the subclass label bits of class $i$. Finally, the summation of class and subclass label information bits, i.e., equations \ref{eq8} and \ref{eq9} completes the proof of Theorem \ref{th3}.
\end{proof}
\begin{figure} [htp] 
    \centering
    \includegraphics[width=1\linewidth]{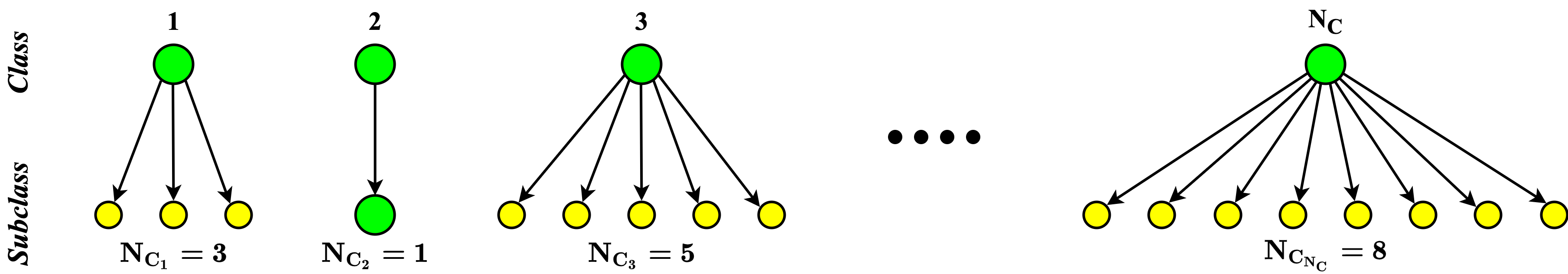}
    \caption{Example of a class hierarchy in a multiclass classification task with a different number of subclasses per class.}
    \label{fig5}
\end{figure}

It is important to note that Theorem \ref{th1} is generalizable as long as the normalized confusion matrix for the training set follows the structural pattern of the confusion matrix for Strong Symmetric Channel or Weakly Symmetric Channel. The following remark illustrates this.

\begin{remark} \label{rem2}
Let the class-level and/or subclass-level normalized confusion matrix for the training set follow either the structural pattern of the Strong Symmetric Channel's or Weakly Symmetric Channel's confusion matrix. Then, the capacity of a strong symmetric channel or a weakly symmetric channel tells us how many label bits the teacher can transfer to the student via class and/or subclass labels. Thus, if we substitute the following capacity for the capacity of the Q-ary Symmetric Channel, all the preceding results of Theorems \ref{th1} hold true.
\begin{flalign} \label{eq10}
 C = \log (N) - H (\text{row of normalized confusion matrix})
\end{flalign}
where $N$ denotes the number of classes or subclasses. Note that the capacity is achieved by a uniform distribution on the class or subclass space. 
\end{remark}

\section{Detailed Analysis of our Experimental Results for the Number of Label Information Bits}

In the paper, we have already shown that a student trained through the SKD framework can benefit from using additional label information. We quantify this extra knowledge by utilizing the proposed upper bound in Theorem $1$ of the paper. To analyze the label bits that the teacher can transfer to the student, we take the best of the $80$ runs based on the binary class F-$1$ score for the training set. The following table details the label bit measurement process by computing the parameters and variables associated with the upper bound. These parameters and variables are determined using the confusion matrix for the training set, which is reused during the transferring phase.

\begin{table}[htb] 
\centering
\begin{tabular}{c c c c c c c c c c}
\hline
{Task} & {$P_{H_{0}}$} & {$P_{H_{1}}$} & {$P_{H_{00}}$} & {$P_{H_{11}}$} & {$K(P_{H_{0}}, P_{H_{1}})$} & ${\alpha^{*}}$ & {Class label bits} & {Subclass label bits} & {Total label bits} \\
\hline
{CL-$11$} & $1.00$ & $0.94$ & - & - & $0.3483$ & $0.4680$ & $0.8363$ & - & $0.8363$\\ 
{SL-$21$} & $0.99$ & $0.93$ & - & $0.96$ & $0.2738$ & $0.4769$ & $0.7915$ & $0.3749$ & $1.1664$\\ 
{SL-$41$} & $1.00$ & $0.82$ & - & $0.86$ & $0.8294$ & $0.4392$ & $0.6441$ & $0.5243$ & $1.1684$\\ 
{SL-$22$} & $0.99$ & $0.80$ & $0.97$ & $0.91$ & $0.8116$ & $0.4467$ & $0.5781$ & $0.6977$ & $1.2758$\\ 
{SL-$12$} & $1.00$ & $0.96$ & $0.97$ & - & $0.2524$ & $0.4754$ & $0.8793$ & $0.4226$ & $\mathbf{1.3019}$ \\
{SL-$14$} & $0.94$ & $0.89$ & $0.84$ & - & $0.2078$ & $0.4868$ & $0.5849$ & $0.5707$ & $1.1556$\\ 
\hline
\end{tabular}
\caption{The proposed upper bound's parameters and variables for our experimental tasks on the MHIST dataset. Total label bits per sample is the sum of the class and subclass label bits per sample. (SL: SubclassLevel, CL: ClassLevel)}
\label{table4}
\end{table}

As shown in Table \ref{table4}, the teacher that was trained in the SL-$12$ task can gain $0.4656$ label information bits per sample from hidden subclass knowledge, increasing its generalization ability. Furthermore, this fine-grained knowledge can be used to improve the student performance via the SKD framework.

\section{Explanation on the Efficiency of Subclass Knowledge Distillation}

The teacher network that trains on subclass labels learns more fine-grained features in comparison to the model trained on the class labels. This can help the teacher to generalize better and, consequently, boost the student performance through the SKD framework. For example, we have shown in the paper that the HP-hard and HP-easy subclasses have some similar fine-grained features that can enhance the teacher's generalization tendencies and then distill more dark knowledge to the student. This knowledge, we explained, causes the student network to bias the incorrect prediction probability of the HP-hard subclass towards the HP-easy subclass rather than SSA. When the evaluation is on the class labels, the student can benefit from this knowledge as long as the temperature hyperparameter is set to a sufficiently high value. To ensure transparency, we show the biasing behaviour of error rate probabilities and learning more features in the following figures for each task on which we ran an experiment. 

\begin{figure} [htp] 
    \centering
    \includegraphics[width=1\linewidth]{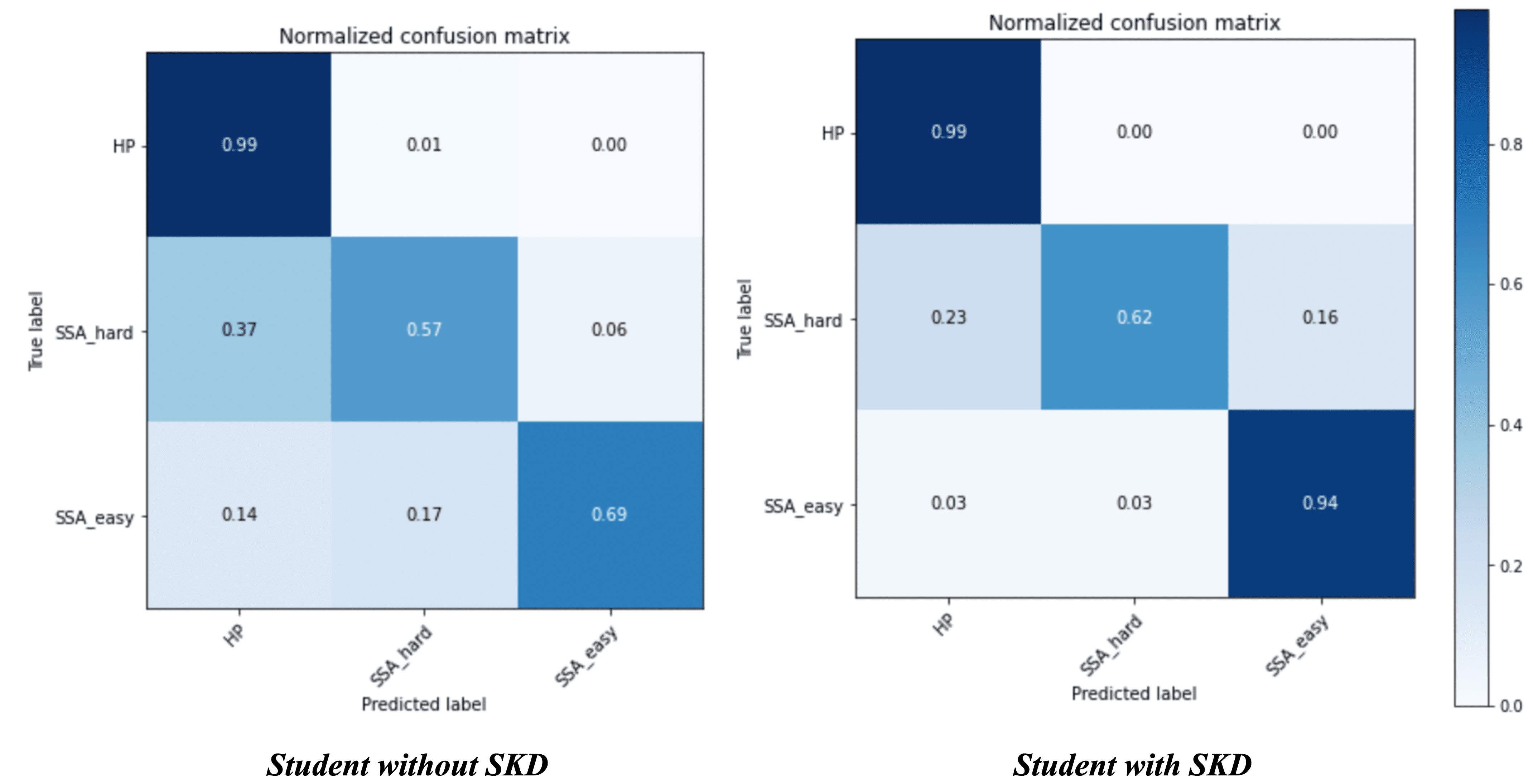}
    \caption{Normalized confusion matrix on the test set for student trained in the SL-$21$ task with and without SKD.}
    \label{fig6}
\end{figure}

\begin{figure} [htp] 
    \centering
    \includegraphics[width=1\linewidth]{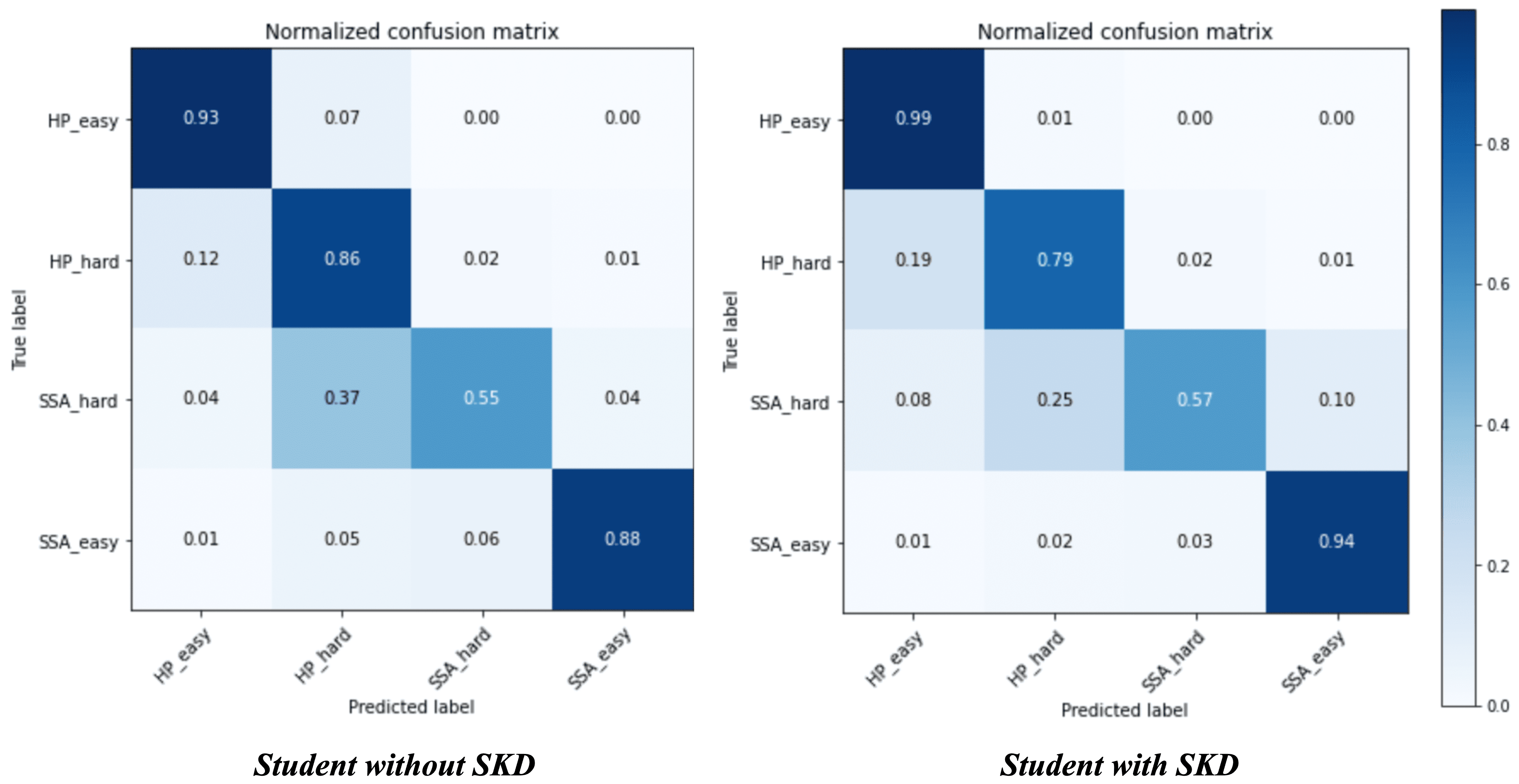}
    \caption{Normalized confusion matrix on the test set for student trained in the SL-$22$ task with and without SKD.}
    \label{fig7}
\end{figure}
\begin{figure} [htp] 
    \centering
    \includegraphics[width=1\linewidth]{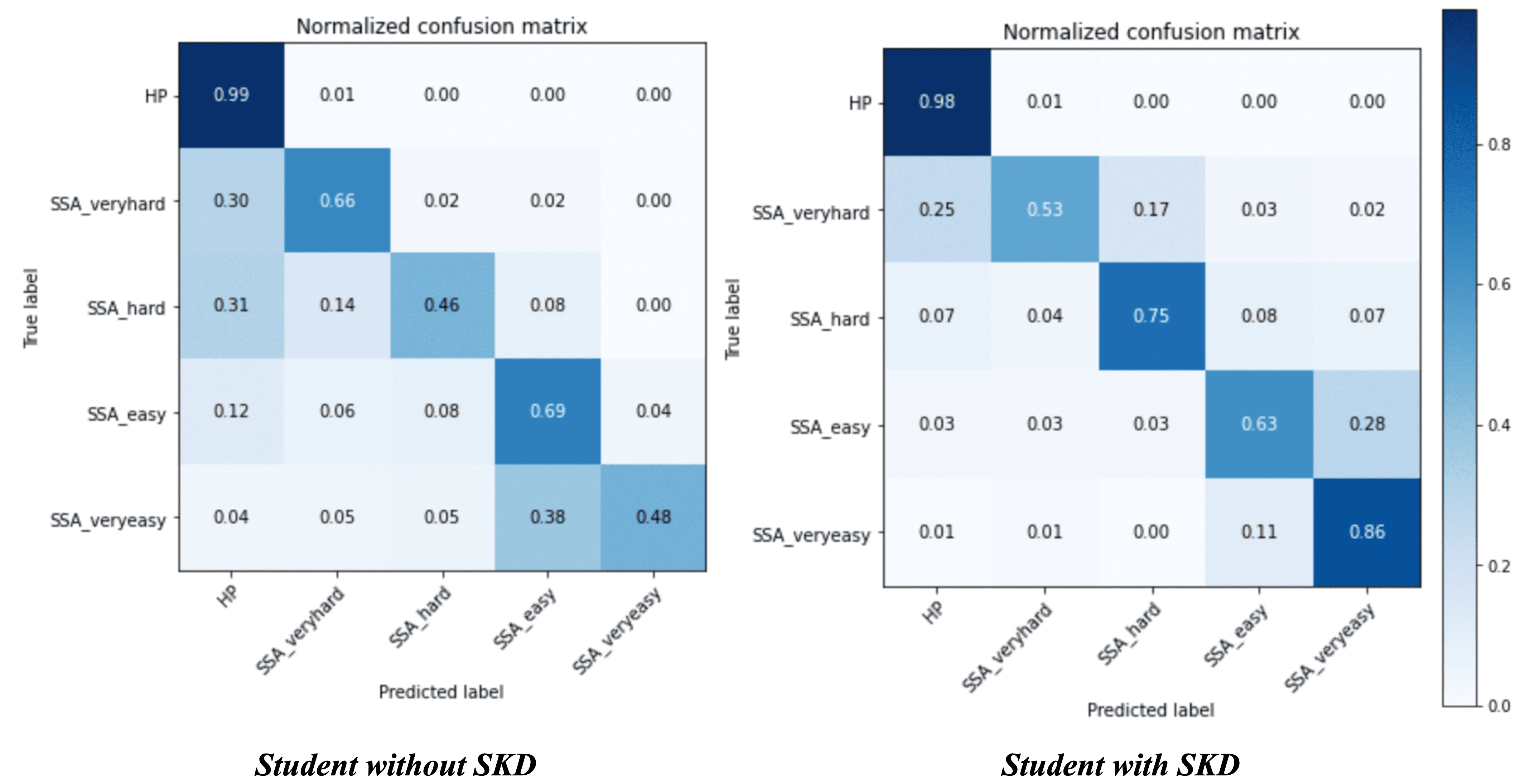}
    \caption{Normalized confusion matrix on the test set for student trained in the SL-$41$ task with and without SKD.}
    \label{fig8}
\end{figure}

\begin{figure} [htp] 
    \centering
    \includegraphics[width=1\linewidth]{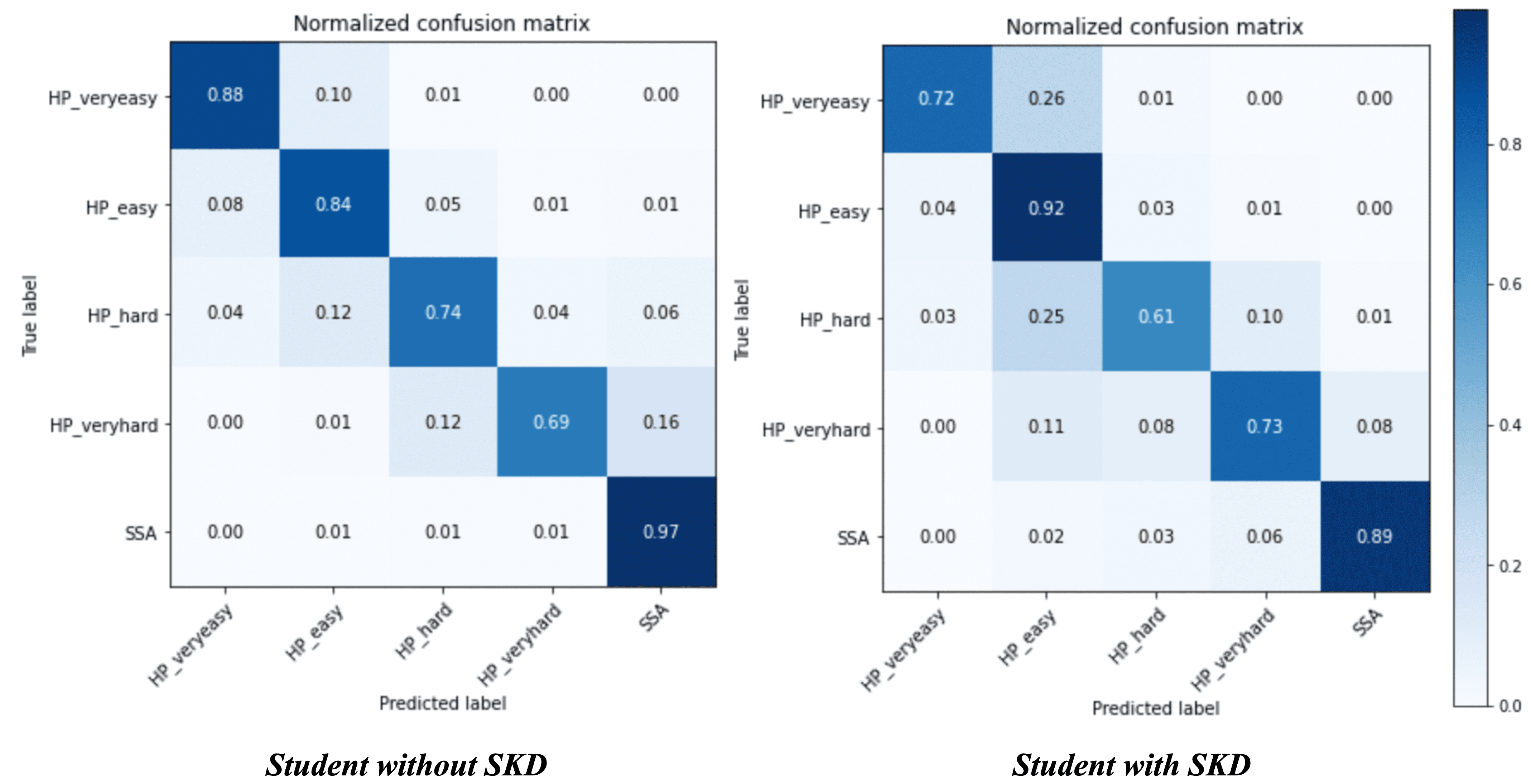}
    \caption{Normalized confusion matrix on the test set for student trained in the SL-$14$ task with and without SKD.}
    \label{fig9}
\end{figure}

\end{document}